%% file: main.tex
\documentclass[letterpaper, 10 pt, conference]{ieeeconf}  
\usepackage{hyperref}
\usepackage{stfloats}
\usepackage{graphicx}
\usepackage{caption}
\usepackage{paralist}
\usepackage{color}
\usepackage{amsmath,amssymb}
\usepackage{float}
\usepackage{booktabs}

\newtheorem{theorem}{Theorem}[section]

\newtheorem{lemma}[theorem]{Lemma}

\IEEEoverridecommandlockouts                              

\title{\LARGE \bf
SARA-RT: Scaling up Robotics Transformers with \\ Self-Adaptive Robust Attention
}

\author{
Isabel Leal$^{*}$, Krzysztof Choromanski$^{*}$, Deepali Jain$^{*}$, Avinava Dubey$^{*}$, Jake Varley$^{*}$, \\  Michael Ryoo, Yao Lu, Frederick Liu, Vikas Sindhwani,  Quan Vuong, Tamas Sarlos, Ken Oslund, \\ Karol Hausman, Kanishka Rao \\
\small{\textrm{isabelleal, kchoro, jaindeepali, avinavadubey, jakevarley, mryoo, yaolug, frederickliu,}}\\
\small{\textrm{sindhwani, quanhovuong, stamas, kenoslund, karolhausman, kanishkarao; @google.com}} \\
\large{Google} \\ 
\small{1600 Amphitheatre Parkway} Mountain View, CA  94043, *equal contribution
}

\begin{document}

\twocolumn[{%
\renewcommand\twocolumn[1][]{#1}%
\maketitle
\begin{center}
    \centering
    \captionsetup{type=figure}
    \includegraphics[width=0.99\textwidth]{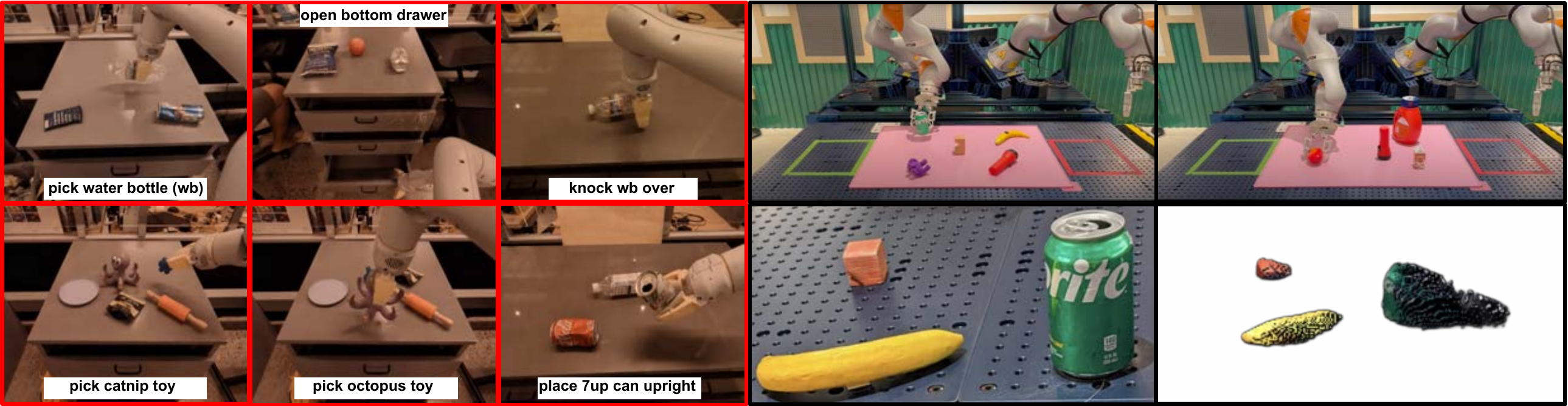}
    \captionof{figure}{\small{Robotics Transformer policies obtained via \textit{Self-Adaptive Robust Attention} (SARA) in action for three different modalities: vision, language and point clouds and varying sequence lengths (from $\sim 200$ to few thousand). \textbf{Left:} A $5$B vision-language-action model from the RT-2 class \cite{brohan2023rt2} (sequence length: $L=196$). The manipulation policy is conditioned on the text instruction. \textbf{Right:} A Point Cloud Transformer (PCT) (\cite{pct}) for grasping (point cloud sizes $L \in [800, 4000]$) and an exemplary PC input. Both policies use linear attention up-trained with SARA and provide computational speedups while maintaining high quality.} \label{fig:intro}}  
\end{center}
}] 

\begin{abstract}
We present Self-Adaptive Robust Attention for Robotics Transformers (SARA-RT): a new paradigm for addressing the emerging challenge of scaling up Robotics Transformers (RT) for on-robot deployment. SARA-RT relies on the new method of fine-tuning proposed by us, called \textit{up-training}. It converts pre-trained or already fine-tuned Transformer-based robotic policies of quadratic time complexity (including massive billion-parameter vision-language-action models or VLAs), into their efficient linear-attention counterparts maintaining high quality. We demonstrate the effectiveness of SARA-RT by speeding up: (a) the class of recently introduced RT-2 models \cite{brohan2023rt2}, the first VLA robotic policies pre-trained on internet-scale data, as well as (b) Point Cloud Transformer (PCT) robotic policies operating on large point clouds. We complement our results with the rigorous mathematical analysis providing deeper insight into the phenomenon of SARA. 
\end{abstract}

\input{intro_related}

\input{sara}
\input{math}

\input{exps}

\input{conclusion}


\bibliographystyle{IEEEtran} 
\bibliography{IEEEfull,main}


\end{document}

%% file: intro_related.tex
\section{INTRODUCTION \& RELATED WORK}
The unprecedented semantic reasoning capabilities offered by Transformers (\cite{vaswani}, \cite{palm}, \cite{gpt},  \cite{weitay}, \cite{clip}, \cite{flam}, \cite{pali}, \cite{socratic}) for Robotics gave rise to the new field of machine learning, exploring Transformer-based models, often pre-trained on the internet-scale data, for robotic controllers. Those controllers enable abstract reasoning from the multi-modal input and have already led to several recent breakthroughs in Robotics, including: high-level planning with large language models (LLMs) (\cite{IchterBCFHHHIIJ22}, \cite{zengichter}, \cite{ren2023robots}, \cite{huangxia}), policies transforming natural language commands into on-robot executable code (\cite{liangzeng}, \cite{zengichter}), multi-modal sensor fusion \cite{palme}, finally the first vision-language-action robotic manipulation powered by massive vision language models \cite{brohan2023rt2}. Interestingly, even when not fine-tuned, Transformer models trained on massive web corpus seem to learn structural reinforcement learning priors that can be leveraged to conduct trajectory optimization for certain control tasks via in-context learning (\cite{oswald}, \cite{baichen}) with carefully designed prompts, as very recently shown in \cite{mirchandani2023large}.

Is the Transformer-driven revolution in Robotics a straightforward path to AGI and all that is left is to train larger and deeper models, pre-trained on even more massive datasets ? One of the challenges that is still not addressed, yet is of critical practical importance in Robotics, is a prohibitively expensive space and time complexity of the aforementioned models. For example, the $35\textrm{M}$-size parameter RT-1 model \cite{rt1}, providing dramatic generalization improvements over non-Transformer approaches and marking the birth of the class of Robotics Transformer architectures (RTs), even though of modest size, already operates with the frequency at most $3 \textrm{Hz}$. The problem only deepens for larger 1B+ parameter models such as RT-2 \cite{brohan2023rt2}. 

\begin{figure*}[h]
    \begin{center}
    \includegraphics[width=.99\linewidth]{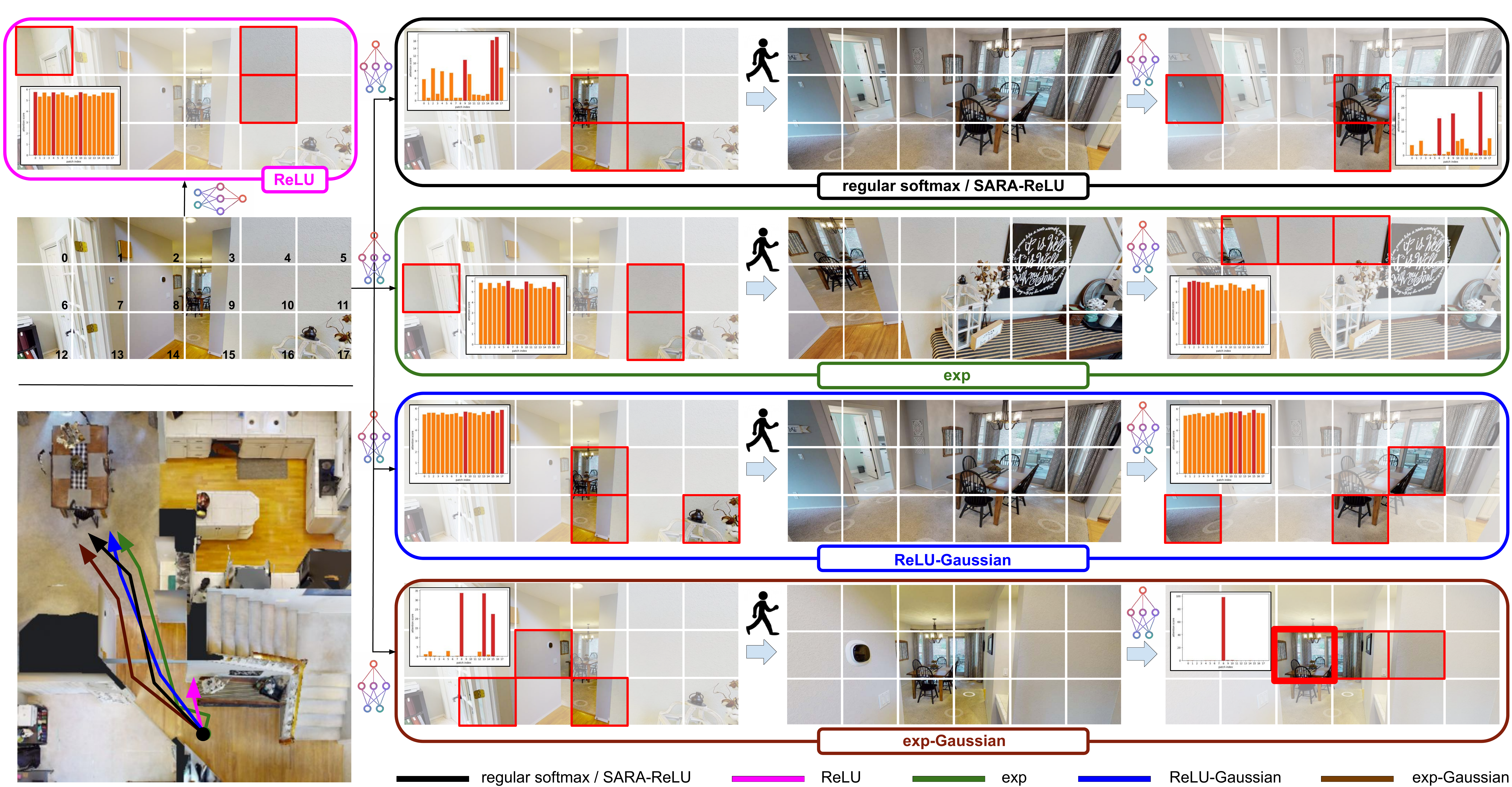}
    \caption{\small{VR navigation via VL attention models on \textrm{Matterport} environments (\cite{matterport}). The top-down view of the scene is in the lower-left corner. The agent's initial view with indexed image patches is right above. For each view, we highlight three patches with highest attention scores and present the distribution of the attention scores over all the patches. At any given point of time, the agent chooses one of them with probability proportional to the score. The initial parts of the executed trajectories are visualized. The trajectories are also included in the top-down view. The agent is given a text command: \textit{"Go to the table"} (or the corresponding image). The ReLU-variant (pink block) directly leads to the collision with the wall. For the exp-variant the agent can navigate, but is temporarily distracted by the unrelated object in the lower-right corner of the initial view. Both variants with Gaussian-matrix pre-processing (blue and brown blocks) apply $m=2048$ (patch/text embedding dimensionality for CLIP with ViT-B is $d=512$) and lead to efficient navigation. The attention scores distribution for the ReLU-Gaussian remains flat though. The SARA variant with $m=d$, and where matrix $\mathbf{G}$ is trained to mimic regular softmax-attention for every new view, produces indistinguishable (spiky) attention scores and leads to efficient navigation (black block). In practical applications involving Transformers, SARA-training is conducted once in the so-called \textit{up-training} process (see: Sec. \ref{sec:sara_birth}). \label{fig:VR_navigation} \vspace{-6mm}}}
    \end{center}
\end{figure*}

In this work we make a step towards solving it. We present \textit{Self-Adaptive Robust Attention for Robotics Transformers} (SARA-RT): a new paradigm for addressing the emerging challenge of scaling up Robotics Transformers for on-robot deployment. SARA-RT relies on the new method of fine-tuning proposed by us, called \textit{up-training}. It leverages it to convert pre-trained or already fine-tuned Transformer-based robotic policies of quadratic space and time complexity (including massive billion-parameter vision-language-action models or VLAs), into their efficient linear-attention counterparts maintaining high quality. We demonstrate the effectiveness of SARA-RT by speeding up: (a) the class of the aforementioned RT-2 models, the first VLA robotic policies pre-trained on internet-scale data, as well as (b) Point Cloud Transformer (PCT) \cite{pct} robotic policies operating on large point clouds (see: Fig. \ref{fig:intro}). We complement our results with the rigorous mathematical analysis explaining the phenomenon of the effectiveness of SARA-RT's.

%% file: sara.tex
\section{SELF-ADAPTIVE ROBUST ATTENTION VIA THE KERNELIZATION VIEWPOINT}

\subsection{Developing intuition: zero-shot navigation via VL models}
\label{sec:intuition}

Consider a vision-based VR navigation agent, conditioned on the images of the target objects: $t_{1},...,t_{M}$ or the corresponding natural language commands (e.g. \textit{"Go to the table"}). For simplicity, we will assume that the targets do not need to be visited in any particular order. First we show that vision-language (VL) models can be used in a zero-shot manner for steering the agent. We take the VR environments from the \textrm{Matterport} (\cite{matterport}) real estate virtual tour sites. 

We consider a purely zero-shot attention-based control mechanism, where the action $\mathbf{a}_{i}$ of the agent corresponding to the particular target $t_{i}$ ($i=1,...,M$) is defined as follows:
\begin{equation}
    \begin{cases}
\mathbf{a}_{i}= \sum_{j=1}^{N} s(i,j)\bar{\mathbf{a}}_{j}, \\
        s(i,j)=\frac{\mathrm{K}(\mathbf{q}_{i},\mathbf{k}_{j})}{\sum_{l=1}^{N}\mathrm{K}(\mathbf{q}_{i},\mathbf{k}_{l})}
    \end{cases}
\end{equation}
and: \textbf{(1)} $N$ stands for the number of the patches of the \textit{patchified} version of the input image, \textbf{(2)} $\{\mathbf{k}_{i}\}_{i=1,...,N} \subseteq \mathbb{R}^{d_{QK}}$ is the set of their latent embeddings (keys), \textbf{(3)} $\mathbf{q}_{i} \in \mathbb{R}^{d_{QK}}$ is the embedding of the target image (or the corresponding text command) (query), \textbf{(4)} $\mathrm{K}:\mathbb{R}^{d_{QK}} \times \mathbb{R}^{d_{QK}} \rightarrow \mathbb{R}$ is the kernel function (similarity measure) defining how the agent \textit{attends} and \textbf{(5)} $\mathbf{\bar{a}}_{j}$ are \textit{base actions} (e.g. left/right turns, going forward/back). We call the coefficients $s(i,j)$ the \textit{attention scores}. We inherit the embeddings from the pre-trained \textrm{CLIP} model \cite{clip} with the ViT-B vision backend.

\textrm{CLIP} model as a zero-shot navigation agent is presented in Fig. \ref{fig:VR_navigation} (top black block). As a baseline, we applied the \textit{softmax-kernel} $\mathrm{K}(\mathbf{x},\mathbf{y})\overset{\mathrm{def}}{=}\exp(\mathbf{x}^{\top}\mathbf{y})$, a default choice for attention. To visualize the main concepts, we interpret $\mathbf{a}_{i}$ as the expected base action over discrete probabilistic distribution with probabilities $(s(i,j))_{j=1,...,N}$ and sample a base action from its variant truncated to the three top-score actions. The base action corresponds to "clicking" on the random pixel of the corresponding patch. As we see, the baseline agent can reach the target without any additional task-specific fine-tuning of the \textrm{CLIP} embeddings.

Now consider the hypothetical case when the set of targets is large, even comparable with $N$. In that case, quadratic space and time complexity $O(MN)$ might become prohibitively expensive. This can be addressed algorithmically if $\mathrm{K}$ admits a bi-linearization, i.e. can be re-written as a linear (dot-product) kernel in the new input-space:
\begin{equation}
\label{eq:linear}
\mathrm{K}(\mathbf{x},\mathbf{y})=\mathbb{E}[\phi(\mathbf{x})^{\top}\phi(\mathbf{y})]    
\end{equation}
for some (randomized) map: $\phi:\mathbb{R}^{d_{QK}} \rightarrow \mathbb{R}^{m}$. In that case, each $\mathbf{a}_{i}$ can be approximated as: 
\begin{equation}
    \begin{cases}
\mathbf{\tilde{a}}_{i}= \frac{\Psi \phi(\mathbf{q}_{i})}{\Gamma \phi(\mathbf{q}_{i})}, \\
\Psi =  \sum_{j=1}^{N}\mathbf{\bar{a}}_{j}\phi^{\top}(\mathbf{k}_{j}),\\
\Gamma = \sum_{j=1}^{N}\phi^{\top}(\mathbf{k}_{j})
\end{cases}
\end{equation}
and thus all $\mathbf{\tilde{a}}_{i}$ can be computed in time \& space $O(M+N)$ (by computing $\Psi$ and $\Gamma$ first). Linear attention mechanisms is an area of active research (\cite{performers}, \cite{rfattention}, \cite{flurka}, \cite{cvrs}, \cite{schlag}, \cite{irie}, \cite{fplusplus}, \cite{favorsharp}, \cite{cosformer}), yet most of the theoretical focus is on developing low-variance variants unbiasedly approximating original softmax-kernel attention (known as particularly expressive due to its combinatorial nature with spiking attention patterns). Those apply random $\phi$ via random Gaussian projections. This computational overhead makes them practically attractive only for $M, N$ large enough (usually $\mathrm{4K+}$). Besides, linear attention usually produces some performance gap as compared to its brute-force softmax counterpart.  

On the other side of the spectrum are linear attention models with $\phi=\phi_{f}$ given as: $\phi_{f}(\mathbf{z})=(f(z_{1}),...,f(z_{d_{QK}}))^{\top}$ for some $f:\mathbb{R} \rightarrow \mathbb{R}$, where $f$ is an easy to compute function, such as: $\mathrm{ReLU}$ or $\mathrm{exp}$. Those are in practice very fast and provide computational gains for $M, N$ as small as $128$, but are well known for being less accurate than the variants providing unbiased estimation. For that reason, they were not of much theoretical interests.

At first glance, we confirm these findings for the zero-shot VL navigation. The results for $\phi_{f}$ with $f=\mathrm{ReLU}$ and $f=\mathrm{exp}$ are presented in Fig. \ref{fig:VR_navigation} in the pink- and green-border box respectively.
The \textrm{ReLU}-variant leads to the agent hitting the wall and the \textrm{exp}-variant, even though reaches the target, is distracted from it at some point, attending to the unrelated object. Finally, none of them provides distinctive attention spikes, characteristic for the regular softmax-variant (sharply attending to the most relevant image regions), but a relatively flat distribution of the attention scores.  

\subsection{Improving $\phi_{f}$ for zero-shot navigation via randomization}

There exists however a very simple trick that improves $\phi_{f}$, making our agent a more efficient zero-shot navigator. It suffices to preprocess the input to $\phi_{f}$ by the random matrix. To be more specific, we define the randomized version of $\phi_{f}$ as: $\phi^{\mathrm{rand}}_{f}(\mathbf{z})=f(\mathbf{G}\mathbf{z})$, where $\mathbf{G} \in \mathbb{R}^{m \times d_{QK}}$ is a Gaussian matrix with entries taken i.i.d from $\mathcal{N}(0,1)$ (sampled once and used for all inputs $\mathbf{z}$) and $f$ is applied element-wise.

As we show in Fig. \ref{fig:VR_navigation} (blue- and brown-border boxes), this modification enables both the $\mathrm{ReLU}$ and $\mathrm{exp}$ variants to reach their targets with no distractions and furthermore already leads to spiking attention patterns for the $\mathrm{exp}$ variant. This will become more clear later. In Sec. \ref{sec:theory}, we show that  $\phi_{\mathrm{exp}}^{\mathrm{random}}$ is intrinsically related to the unbiased estimators of the original softmax-attention. 

\subsection{The birth of SARA: learnable pre-processing}
\label{sec:sara_birth}

Randomized mappings $\phi^{\mathrm{rand}}_{f}$ in Fig. \ref{fig:VR_navigation} apply $m=2048$ and thus, as we discussed before, may not be relevant for tasks with $M, N < 1K$. The core idea behind \textit{Self-Adaptive Robust Attention} is that rather than being Gaussian, matrix $\mathbf{G}$ can be learned. We define SARA mapping $\phi^{\mathrm{SARA}}_{f}$, acting on the raw $d$-dimensional embeddings: $(\mathbf{x}_{i})_{i=1}^{M}$, $(\mathbf{y}_{j})_{j=1}^{N}$ of objects/tokens (rather than queries/keys $(\mathbf{q}_{i})_{i=1}^{M}/(\mathbf{k}_{j})_{j=1}^{N}$ that for instance for Transformers are obtained from the former by linear projections $W_{Q}$ and $W_{K}$) as follows for learnable $\mathbf{v} \in \mathbb{R}^{m}$, $\mathbf{G}_{Q}, \mathbf{G}_{K}\in \mathbb{R}^{m \times d}$ and a \textit{Hadamard product} $\odot$: 
\begin{equation}
\phi^{\mathrm{SARA}}_{f, 1}(\mathbf{z})=  \mathbf{v} \odot f(\mathbf{G}_{Q}\mathbf{z}), \textrm{     }
\phi^{\mathrm{SARA}}_{f, 2}(\mathbf{z})=  \mathbf{v} \odot f(\mathbf{G}_{K}\mathbf{z})
\end{equation}
The kernel values $\mathrm{K}(\mathbf{x}_{i},\mathbf{y}_{j})$ (for: $\mathrm{K}:\mathbb{R}^{d} \times \mathbb{R}^{d} \rightarrow \mathbb{R}$) are then expressed as dot-products $\phi^{\mathrm{SARA}}_{f, 1}(\mathbf{x}_{i})^{\top}\phi^{\mathrm{SARA}}_{f, 2}(\mathbf{y}_{j})$, as in Eq. \ref{eq:linear}. Instead of one mapping $\phi$, now there are two: $\phi_{1},\phi_{2}$, but we have perfect equivalence since in Eq. \ref{eq:linear}, $\mathbf{x}/\mathbf{y}$ corresponds to the query/key $\mathbf{q}_{i}/\mathbf{k}_{j}$ obtained by two different mapping from $\mathbf{x}_{i}/\mathbf{y}_{j}$. Furthermore, the analogous computational complexity analysis follows.

It turns out that then one can take $m=d_{QK}$, making SARA a practical mechanism even for $M, N < 1K$. And indeed, even though $M, N > 1K$ for our point clouds empirical results in Sec. \ref{sec:pct-exps}, the speed improvements for the RT-2 from Sec. \ref{sec:exp-rt2} were obtained for $M, N < 200$. As showed in Fig. \ref{fig:VR_navigation}, with learnable $\mathbf{G}$ and $m=d_{QK}$ (here $d=d_{QK}$), one can learn linear attention mechanisms that for all practical purposes produce indistinguishable attention patterns as compared with the regular softmax-attention, thus leading to the same navigation trajectories for our VR tasks. 

The VL navigation is a convenient "macroscopic" case study, but our main targets are Transformer-architectures for Robotics, where queries for the whole images or text instructions are replaced by their counterparts for image-patches, text-tokens or even individual point of the point clouds (PCs), produced by the Transformer-encoded policies. 

We thus propose the process of the self-adaptation of their attention modules, that we refer to as \textit{up-training}, which can be implemented as replacing regular softmax-attention with its $\phi_{f}$-encoded variants and fine-tuning them on data from the downstream robotic tasks. Learnable pre-processing corresponds here to fine-tuning matrices $\mathrm{W}_{Q}$ and $\mathrm{W}_{K}$ (\cite{vaswani}) from Transformers' attention modules, but in the linear attention context.  

%% file: math.tex
\section{THE MATHEMATICS OF SARA-RTs}
\label{sec:theory}

As a warm-up, we show that a linear attention mechanism using $\phi^{\mathrm{random}}_{\mathrm{exp}} : \mathbb{R}^{d_{QK}} \rightarrow \mathbb{R}^{m}$ leads to the unbiased estimation of the softmax-kernel up to the fixed constant if the inputs to the softmax-kernel have fixed length (e.g. \textrm{CLIP} embeddings that are by default $L_{2}$-normalized).

\begin{lemma}
\label{lemma:base}
The following holds for $\mathbf{x},\mathbf{y} \in \mathbb{R}^{d_{QK}}$ with $\|\mathbf{x}\|=\|\mathbf{y}\|=r$ and softmax-kernel $\mathrm{K}:\mathbb{R}^{d_{QK}} \times \mathbb{R}^{d_{QK}} \rightarrow \mathbb{R}$:
\vspace{-5mm}
\begin{equation}
\label{eq:unbiased}
m \cdot \exp(r^{2}) \cdot \mathrm{K}(\mathbf{x},\mathbf{y}) = \mathbb{E}[(\phi^{\mathrm{random}}_{\mathrm{exp}}(\mathbf{x}))^{\top}\phi^{\mathrm{random}}_{\mathrm{exp}}(\mathbf{y})]    
\end{equation}
\end{lemma}
\begin{proof}
Note that the \textit{positive random feature} vector $\phi^{+}(\mathbf{z})$ from \cite{performers} is of the following form for any $\mathbf{z} \in d_{QK}$:
\begin{equation}
\phi^{+}(\mathbf{z}) = \frac{1}{\sqrt{m}}\exp(-\frac{\|\mathbf{z}\|^{2}}{2})\exp(\mathbf{G}\mathbf{z})   
\vspace{-1.5mm}
\end{equation}
By Lemma 1 from \cite{performers}: $\mathrm{K}(\mathbf{x},\mathbf{y})=\mathbb{E}[\phi^{+}(\mathbf{x})^{\top}\phi^{+}(\mathbf{y})]$.  That implies Eq. \ref{eq:unbiased} and completes the proof.
\end{proof}
One can also derive corresponding concentration results.
\begin{lemma}
Given the conditions from Lemma \ref{lemma:base}, $g^{m}_{r,t}=\frac{t}{\sqrt{m}}\exp(r^{2}(2\cos(\theta)+1))\sqrt{(1-\exp(-2r^{2}(1+\cos(\theta))))}$ for an angle $\theta$ between $\mathbf{x},\mathbf{y}$, the following holds for $t>0$:
\begin{equation}
\mathbb{P}[|\frac{(\phi^{\mathrm{random}}_{\mathrm{exp}}(\mathbf{x}))^{\top}\phi^{\mathrm{random}}_{\mathrm{exp}}(\mathbf{y})}{m \exp(r^{2})} - \mathrm{K}(\mathbf{x},\mathbf{y})| > g^{m}_{r,t}] \leq \frac{1}{t^{2}}    
\end{equation}
\end{lemma}
\begin{proof}
By Lemma 2 from \cite{performers}, the variance of the estimator $\widehat{\mathrm{K}}(\mathbf{x},\mathbf{y})=\phi^{+}(\mathbf{x})^{\top}\phi^{+}(\mathbf{y})$ satisfies: $\mathrm{Var}(\widehat{\mathrm{K}}(\mathbf{x},\mathbf{y}))=\frac{1}{m}\exp(-(\|\mathbf{x}\|_{2}^{2}+\|\mathbf{y}\|_{2}^{2}))(\exp(2\|\mathbf{z}\|_{2}^{2})-\exp(\|\mathbf{z}\|_{2}^{2}))$ for $\mathbf{z}=\mathbf{x}+\mathbf{y}$. To complete the proof, it suffices to note that
$\widehat{\mathrm{K}}(\mathbf{x},\mathbf{y})=\frac{(\phi^{\mathrm{random}}_{\mathrm{exp}}(\mathbf{x}))^{\top}\phi^{\mathrm{random}}_{\mathrm{exp}}(\mathbf{y})}{m \exp(r^{2})}$, $\|\mathbf{z}\|^{2}_{2}=2r^{2}(1+\cos(\theta))$ and apply Chebyshev's Inequality \cite{boucheron}.
\end{proof}

Our main result shows that for the appropriate $\mathbf{v},\mathbf{G}, f$, mappings $\phi^{\mathrm{SARA}}_{f, u}$ ($u=1,2$) lead to the accurate approximation of the attention score matrix $\mathbf{A}=[\mathrm{K}(\mathbf{q}_{i},\mathbf{k}_{j})]_{i=1,...,M}^{j=1,...,N} \in \mathbb{R}^{M \times N}$ in the Transformer, importantly with the number of trainable parameters at most \textbf{logarithmic} in $\mathbf{A}$ size. 

\begin{theorem}
Consider normalized Transformer's attention layer with queries $(\mathbf{q}_{i})_{i=1}^{M}$ and keys $(\mathbf{k}_{j})_{j=1}^{N}$ of length $r$. Denote: $\tau=\min_{i,j} \mathrm{K}(\mathbf{q}_{i},\mathbf{k}_{j})$, $\rho=\max_{i,j} \mathrm{K}(\mathbf{q}_{i},\mathbf{k}_{j})$.  

Take $m = \lceil \frac{2\rho^{2}}{\delta^{2}\tau^{2}}\log(2MN)\exp(-\frac{r^{2}}{A}) \rceil + 1$ for $A<0$, $\delta>0$. Then there exist $\mathbf{v} \in \mathbb{R}^{m},\mathbf{G}_{1},\mathbf{G}_{2} \in \mathbb{R}^{m \times d},f:\mathbb{R} \rightarrow \mathbb{R}$ such that the approximate attention matrix $\widehat{\mathbf{A}}$ (implicitly) given by the mappings $\phi^{\mathrm{SARA}}_{f, u}$ satisfies:  
\vspace{-1.5mm}
\begin{equation}
\|\mathbf{A}-\widehat{\mathbf{A}}\|_{\infty} \leq \delta
\end{equation}
\begin{proof}
Take $f=\exp$. We will show that if: 
\begin{align}
\begin{split}
\mathbf{G}_{1} = \sqrt{1-4A}\mathbf{G}\mathbf{W}_{Q}, \textrm{   } \mathbf{G}_{2}=\sqrt{1-4A}\mathbf{G}\mathbf{W}_{K}, \\ \mathbf{v}=(1-4A)^{\frac{d_{QK}}{4}}(\exp(A \|\mathbf{g}_{1}\|^{2}),...,\exp(A \|\mathbf{g}_{m}\|^{2}))^{\top},
\end{split}    
\end{align}
where $\mathbf{G} \in \mathbb{R}^{m \times d_{QK}}$ is a Gaussian matrix with iid entries from $\mathcal{N}(0,1)$ and $\mathbf{g}_{1},...,\mathbf{g}_{m}$ are its rows, then the resulting $\widehat{\mathbf{A}}$ satisfies the condition with nonzero probability for $m$ given in the statement. Here $\mathbf{W}_{Q}$ and $\mathbf{W}_{K}$ are standard linear projections of the raw embeddings $(\mathbf{x}_{i})_{i=1}^{M}$ and $(\mathbf{y}_{j})_{j=1}^{N}$, defining queries $(\mathbf{q}_{i})_{i=1}^{M}$ and keys $(\mathbf{k}_{j})_{j=1}^{N}$ and we assume that attention matrix entries $\tilde{\mathbf{A}}_{i,j}$ before row-normalization satisfy: $\tilde{\mathbf{A}}_{i,j}=\mathrm{K}(\mathbf{q}_{i},\mathbf{k}_{j})=\exp(\mathbf{q}_{i}^{\top}\mathbf{k}_{j})$ (softmax attention).

Take $\epsilon = \frac{\delta}{\rho}$. Denote by $p_{\epsilon}$ the probability of an event $E(\epsilon)=\{\exists_{i,j} |\widehat{\mathrm{K}}(\mathbf{q}_{i},\mathbf{k}_{j})-\mathrm{K}(\mathbf{q}_{i},\mathbf{k}_{j})| > \epsilon \mathrm{K}(\mathbf{q}_{i},\mathbf{k}_{j})\}$. Note that by Theorem 3.1 from \cite{fplusplus}, $\frac{\phi^{\mathrm{SARA}}_{f,1}(\mathbf{x}_{i})^{\top}\phi^{\mathrm{SARA}}_{f,2}(\mathbf{y}_{j})}{m\exp(r^{2})}$ is an unbiased estimator of $\mathrm{K}(\mathbf{q}_{i},\mathbf{k}_{j})$. Furthermore, by Theorem 4.2. from \cite{fplusplus}, the probability that 
$|\frac{\phi^{\mathrm{SARA}}_{f,1}(\mathbf{x}_{i})^{\top}\phi^{\mathrm{SARA}}_{f,2}(\mathbf{y}_{j})}{m\exp(r^{2})}-\mathrm{K}(\mathbf{q}_{i},\mathbf{k}_{j})| > \tau \epsilon$ is at most $r_{\tau \epsilon}=2\exp(-\frac{m\tau^{2}\epsilon^{2}}{2}\exp(\frac{r^{2}}{A}))$. Thus, by the union bound, we get: $p_{\epsilon} \leq MN r_{\tau \epsilon}$. Note that if $E(\epsilon)$ does not hold, then $|\mathbf{A}_{i,j}-\widehat{\mathbf{A}}_{i,j}| \leq \epsilon \rho$ for every $i,j$. That completes the proof, since for $m$ defined in the statement of the theorem, $p_{\epsilon} < 1$.
\end{proof}
\end{theorem}

%% file: exps.tex
\section{EXPERIMENTS}
\subsection{Robotic Point Cloud Transformers}
\label{sec:pct-exps}

In our first set of experiments, we trained robotic grasping Transformer policies operating on the point cloud (PC) data. Since Point Cloud Transformers (\cite{pct}) usually use relatively long 1K+ sequences, even for simple objects, the unscalability of the brute-force quadratic attention is a severe problem.
\vspace{-5mm}
\begin{figure}[h]
    \begin{center}
    \includegraphics[width=.90\linewidth]{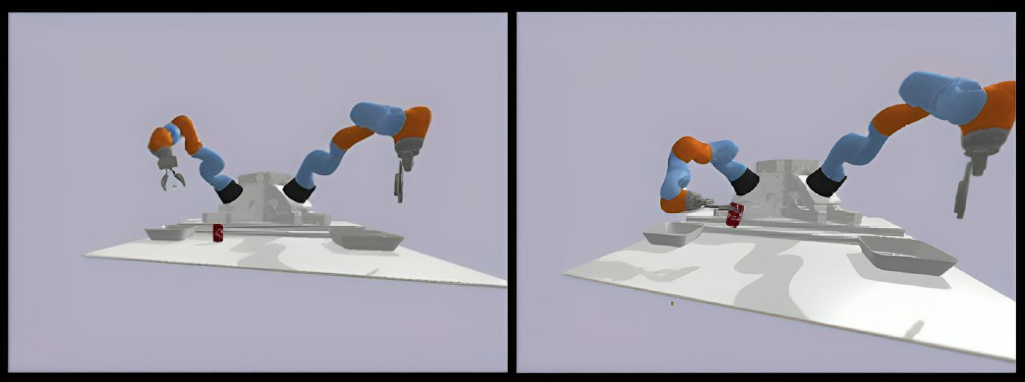}
    \caption{\small{The simulator used to train PC-input grasping policies and the successful coke can grasp with corresponding reward $r=1$.} \label{fig:sim}}
    \vspace{-4mm}
    \end{center}
\end{figure}
\vspace{-4mm}
\begin{figure}[h]
    \begin{center}
    \includegraphics[width=.99\linewidth]{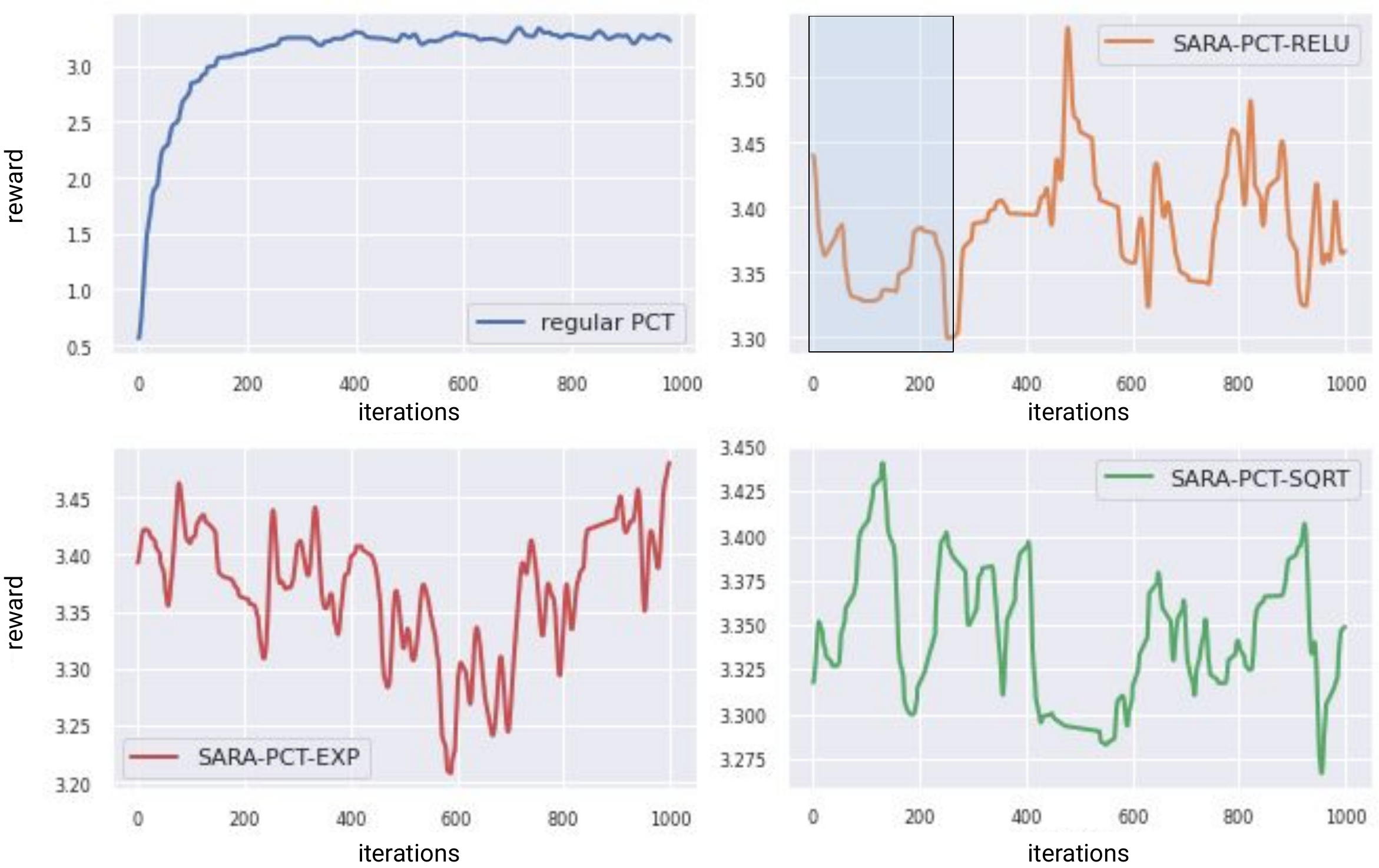} 
    \caption{\small{Training regular PCT policy as well as three variants of SARA with $f \in \{\mathrm{ReLU}, \mathrm{exp}, \mathrm{sqrt}\}$ (up-training from the regular PCT checkpoint). The adaptation process of the linear attention for the ReLU variant is highlighted. For each of the five objects to grasp, an agent receives a binary reward (1: \textrm{grasp}, 0: \textrm{no grasp}). The reported reward is computed as their average over $100$ trials. For all the variants, all the checkpoints (even in the preliminary stage of up-training) provide good quality policies (the $y$-axis for SARA-plots is the magnified version of the short high-reward interval of the y-axis for the regular PCT training.)} \label{fig:uptraining}}
    \vspace{-5mm}
    \end{center}
\end{figure}

\subsubsection{The setting} 
A grasping policy receives a PC from a \textrm{Realsense} camera. A pass-through filter removes all points except for those of table top objects. These points are hierarchically clustered into objects. A single object PC is then passed into a PCT policy to produce a grasp pose offset from the center of the object point cloud. The observation space has 3 components: \textbf{(1)} $cloud$: $N \times 3$ point cloud with the workspace origin at the mean of the object's cloud; \textbf{(2)} $center$: $(x,y,z)$ cloud center in workspace frame of reference; \textbf{(3)} $major\_axis$: $(x, y, z)$ representing the major axis of the points in the object frame. The grasp pose action is represented by the fingertip position relative to the object center, the approach direction vector in the workspace frame, and the robot arm wrist roll angle. The grasp is then executed open-loop with a \textrm{Kuka} \textrm{IIWA} robotic arm and Weiss gripper. 

PCT policy training is conducted in the simulator (Fig. \ref{fig:sim}) via blackbox optimization (\cite{salimans}, \cite{ars}, \cite{isim2real}, \cite{tabletenniscs}). The BGS variant (\cite{isim2real}) with $l=50$ perturbation-directions, Gaussian smoothing parameter $\sigma=0.02$, step size $\eta=0.02$, $\tau=30\%$ top directions is applied. In training, an agent sees only $k=5$ different objects: a coke can, a water bottle, a chalkboard eraser, a banana and an octopus soft toy. Policies learned in sim were ready for the on-robot fine-tuning (also with blackbox methods), but it was not necessary (no major sim-to-real gap was observed). Objects encountered in testing included several of shapes never seen in training.
\begin{figure}[h]
    \begin{center}
    \includegraphics[width=.99\linewidth]{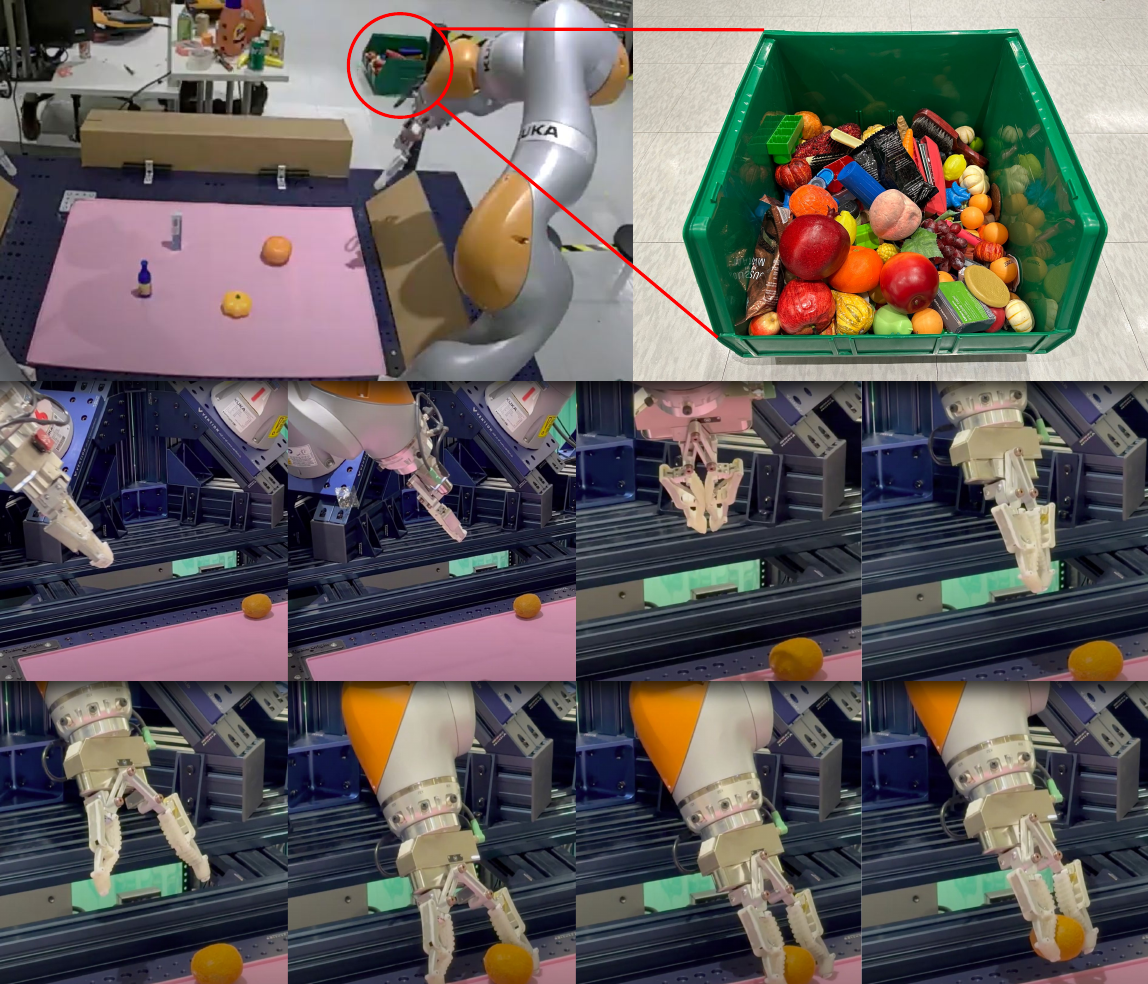}
    \caption{\small{\textbf{Upper row:} The AB-test setup. Different configurations can vary by the number of objects of the table and their shapes. \textbf{Lower rows:} One of the grasps performed by SARA-ReLU in the low-reachability-area setting (target object close to the robotic arm base), requiring careful PC-input-based control of the robotic arm.} \label{fig:ab_test}}
    \vspace{-8mm}
    \end{center}
\end{figure}

\subsubsection{SARA-PCT vs regular PCT} We trained in sim regular PCT policy till convergence for $s=1000$ iterations. In the SARA setting, we then conducted up-training with $f \in \{\exp, \mathrm{ReLU}, \mathrm{sqrt}: x \rightarrow x^{2} \}$, the all-one vector $\mathbf{v}$ and $m=d$. The results are presented in Fig. \ref{fig:uptraining}. We see almost immediate adaptation of the linear attention for all SARA variants. Thus we chose (here and for the RT-2 experiments) the simplest ReLU (that can be thought of as the tamed version of the $\mathrm{exp}$ variant), on-robot deployed it and compared with the regular PCT in the AB-test. The test consisted of $N=200$ random object configurations (see: Fig. \ref{fig:ab_test}), where for each configurations either SARA-PCT or regular PCT policy is randomly selected. The agent gets a binary reward $r \in \{0,1\}$ for each grasp (success or failure). The average reward obtained by the regular PCT agent is: $r^{\mathrm{reg}}_{\mathrm{ave}}=\mathbf{0.64}$ and by the SARA-PCT agent: $r^{\mathrm{SARA}}_{\mathrm{ave}}=\mathbf{0.75}$.

We then run speed tests for the regular PCT encoder and SARA-PCT encoder for input sequences of different lengths. 
Both encoders consist of $t=2$ Transformer layers with $16$-dimensional latent embeddings. 
The actual sequence length for the on-robot deployment varies from scene to scene, but can easily exceed 1K. As showed in Fig. \ref{fig:sara_pct}, SARA-PCT provides significant speedups, guaranteeing practically constant inference $\approx 100\mathrm{ms}$ (regardless of the point cloud size), with the attention module not being a computational bottleneck anymore. On the contrary, for the regular PCT the attention module remains a bottleneck and leads to the substantial slowdowns even for relatively small point clouds. 
\vspace{-5mm}
\begin{figure}[h]
    \begin{center}
    \includegraphics[width=.90\linewidth]{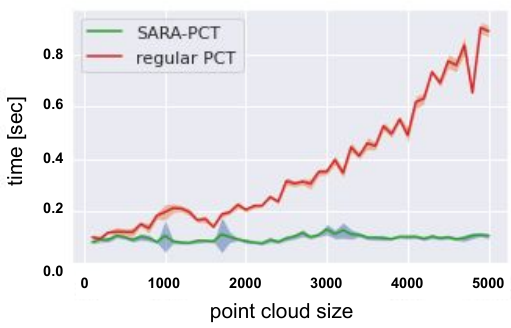} 
    \vspace{-2.0mm}
    \caption{\small{Speed tests for SARA-PCT and regular PCT. Reported are mean inference times (averaged over $l=10$ random seeds) for PCT encoders (as well as the corresponding standard deviations; see: shaded regions) as functions of the point cloud size.} \label{fig:sara_pct}}
    \vspace{-4mm}
    \end{center}
\end{figure}
\vspace{-6.5mm}
\subsection{RT-2 vision-language-action models}
\label{sec:exp-rt2}
\subsubsection{The setting} Next we consider the class of RT-2 architectures from \cite{brohan2023rt2}. Those apply PaLI-X \cite{pali-x}) vision-language-model (VLM) backbones to encode policies taking vision input and conditioned on natural language instructions. We focus on the 5B PaLI-X variant, as more practical for the on-robot deployment than the 55B variant. As described in \cite{brohan2023rt2}, the action space consists of 6-DoF positional and rotational displacement of the robot end-effector, as well as the level of extension of the robot gripper and a special discrete command for terminating the episode, which should be triggered by the policy to signal successful completion. The VLM takes a text instruction and an image (or a history of images) and produces a sequence of text tokens that can be then converted to actions (see: Sec. \ref{sec:tokenizer} for details). 

\begin{figure*}[h]
    \begin{center}
    \includegraphics[width=.99\linewidth]{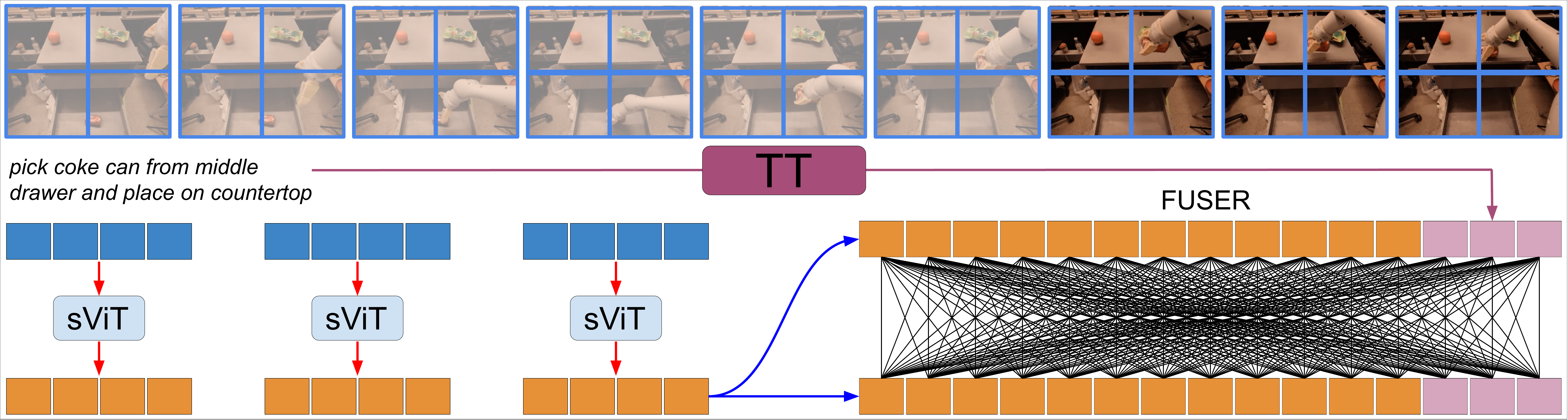}
    \caption{\small{The scheme of some of the key elements of the PaLI-X backbone of RT-2 from the computational viewpoint, accompanied with the real robot performing text instruction using SARA model. This example applies three-frame history with each frame partitioned into four patches (in practice the number of patches is much larger). Frames are encoded via SARA variants of the ViTs (\textrm{sViT}). Text instruction is separately pre-processed by the text Transformer (TT). In the fuser, all resulting embeddings are concatenated and interacting with each other via self-attention. This self-attention block is yet another good candidate for injecting SARA variants. We leave it to future work.} \label{fig:sara_pali}}
    \vspace{-5mm}
    \end{center}
\end{figure*}

\begin{table*}[h]
    \begin{center}
    \caption{\small{Comparison of different variants of the RT-2 models with the PaLI-5B VLM backbones. The comparison is conducted on the six regular manipulation task and an additional task measuring generalization level (\textrm{diverse} \textrm{pick}). The mean accuracies for all three models are: \textbf{65.8\%}, \textbf{65.1\%} and \textbf{76.4\%} (up-down). The first two models differ only by the applied attention module (regular or linear).} \label{table:clocktime}}
    \vspace{-2mm}
    \scalebox{0.95}{
     \begin{tabular}{ c  c  c  c  c  c  c  c} 
    \toprule
      & {$\textbf{\textrm{pick}}$} & {$\textbf{\textrm{knock}}$} & {$\textbf{open/close drawer}$} & {$\textbf{drawer place}$}  & {$\textbf{\textrm{upright}}$}  & {$\textbf{\textrm{move}}$} & \textbf{diverse pick} \\ [0.5ex] 
    \toprule
     RT-2 (PaLI-5B), no history + action tokens & $81\%$  & $86\%$ & $67\%$ & $39\%$ & $57\%$ & $98\%$ & $33\%$\\
     SARA-RT-2 (PaLI-5B), no history + action tokens & $83\%$ & $91\%$ & $78\%$ & $31\%$ & $46\%$ & $79\%$ & $48\%$\\
     $\textbf{SARA-RT-2 (PaLI-5B), history + vector representation}$ & $100\%$ & $91\%$ & $89\%$ & $56\%$ & $51\%$ & $81\%$ & $67\%$\\
    \bottomrule
     number of trials per eval & $36$ & $35$ & $18$ & $36$ & $35$ & $48$ & $21$\\
    \bottomrule
    \end{tabular}}
\end{center}
\vspace{-8mm}
\end{table*}
\normalsize

The pre-trained VLM is fine-tuned on the robotic data from \cite{rt1}. It consists of expert demonstrations collected with a mobile manipulation robot. Each demonstration is mapped to the natural language instruction from one of the following classes: \textit{"pick object"}, \textit{"knock object over"}, \textit{"open/close drawer"} , \textit{"place object into receptacle"}, \textit{"place object upright"}, \textit{"move object near object"}, \textit{"pick object from receptacle and place on the counter"} (details in \cite{rt1}). 

For SARA variants (with $f=\mathrm{ReLU}$ and all-one vector $\mathbf{v}$), up-training is conducted after the fine-tuning phase. There are two natural places to inject linear attention in PaLI-X models: the ViT encoder and the fuser, where tokens from all the frames as well as text ones are concatenated (self-attention modules in both) (see: Fig. \ref{fig:sara_pali}). In this work, we chose the former, leaving testing the latter to future work.

\subsubsection{Action representation} 
\label{sec:tokenizer} Each action can be represented as a vector consisting of: positional and rotational displacements, a scalar tracking gripper's extension level and a sub-vector encoding episode termination command. Traditionally (see: \cite{brohan2023rt2} for details), each continuous dimension is quantized into one of the $256$ bins, leading to an action given by a sequence of integers, tokenized by PaLI-X (the so-called \textit{action tokens}). We propose an alternative  (\textit{vector representation} of the action space), where each continuous dimension is rounded up to four decimal places, leading to the string that can be text-tokenized as usually. We noticed that the vector representation of actions results in higher quality models.
\begin{figure}[h]
    \begin{center}
    \vspace{-3mm}
    \includegraphics[width=.90\linewidth]{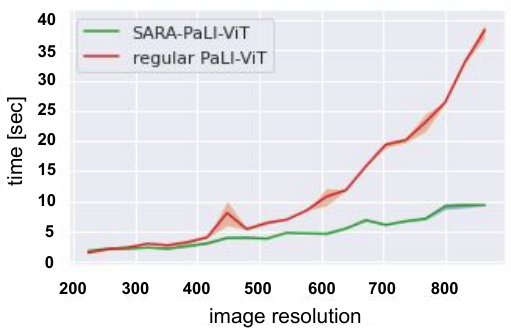} 
    \caption{\small{Speed tests (on a CPU). Reported numbers are as in Fig. \ref{fig:sara_pali}, but for PaLI-ViT encoders as functions of the resolution of the image ($x_{\mathrm{value}} \times x_{\mathrm{value}}$) for the default $16 \times 16$ patch size.} \label{fig:speed-PaLI}}
    \vspace{-4mm}
    \end{center}
\end{figure}
\subsubsection{SARA-RT-2 vs regular RT-2} We start by comparing RT-2 variant using one frame (no history) and old tokenizer with its SARA counterpart. Both use $L=196$ vision tokens. 
As we show in Table \ref{table:clocktime} (first two rows), SARA variant maintains good performance across all seven tested tasks (providing better performance on four of them). Both variants have very similar mean accuracy, though SARA variant generalizes better. Besides, regular RT-2 controller needs $\textbf{53.2}$ ms (TPU) for the forward pass, while SARA's: $\textbf{45.7}$ ms (TPU) ($\textbf{14\%}$ speedup). Finally, we combine SARA with the new tokenizer from \ref{sec:tokenizer} and the history of $H=3$ frames. The latter two techniques provide farther accuracy boost, but are  more demanding computationally, and thus challenging for direct on-robot deployment. It turns out that the resulting ViT-linear-attention hybrid RT-2 variant (third row in Table \ref{table:clocktime}) provides $\textbf{12\%+}$ mean accuracy improvement, excelling in certain tasks (e.g. $\textbf{100\%}$ accuracy on the $\textbf{36}$ \textrm{pick} tasks). 

We run additional speed tests, comparing SARA-based ViT of the RT-2 model's PaLI backbone with its regular counterpart for images of different resolutions (on CPUs). The results are presented in Fig. \ref{fig:speed-PaLI}. The ViT encoder of PaLI is computationally bottlenecked by its attention module. SARA remains a feasible approach even for high resolution images, while the regular variant does not. We plan to exercise this feature of SARA by using much higher resolution images (a challenge for regular RT-2 models) in future work.

%% file: conclusion.tex
\vspace{-1.0mm}
\section{CONCLUSION}
\vspace{-1.0mm}
We introduced SARA-RT, a new paradigm for adapting Transformer-based robotic controllers to their more practical linear-attention deployable counterparts. We have provided its comprehensive evaluation on real robots, including improving the inference of the VLA class of RT-2 models and controllers applying Point Cloud Transformers. We consider it one of the first steps towards addressing practical challenges of using Transformer-based robotic controllers.